\documentclass[journal]{IEEEtran}

\usepackage{latexsym,,amssymb,amsmath,graphicx,epsf,cite,bbm,float}
\usepackage{ifpdf}
\usepackage{epstopdf}

\usepackage{algorithm,hyperref}
\usepackage{amsmath,amssymb,bm}
\usepackage{amsfonts,dsfont,color,bbm,subcaption}

\usepackage[noend]{algpseudocode}

\newcommand{\be}{\begin{equation}}
\newcommand{\ee}{\end{equation}}
\newcommand{\bea}{\begin{eqnarray}}
\newcommand{\eea}{\end{eqnarray}}

\newcommand{\MB}{\left[\begin{array}}
\newcommand{\ME}{\end{array}\right]}

\newcommand{\ei}{\end{itemize}}
\newcommand{\bi}{\begin{itemize}}

\usepackage{amsmath,amsthm}

\newcommand{\E}{\mathbb{E}}
\newcommand\Tau{\mathcal{T}}

\newtheorem{theorem}{Theorem}

\newtheorem{lemma}[]{Lemma}

\newtheorem{proposition}[]{Proposition}

\newtheorem{corollary}[]{Corollary}

\newtheorem{remark}[]{Remark}

\newtheorem{assumption}[]{Assumption}

\begin{document}

\title{Data Dependent Regret Guarantees Against\\ General Comparators for Full or Bandit Feedback} 
\author{Kaan Gokcesu, Hakan Gokcesu}
\maketitle

\flushbottom

\begin{abstract}
	We study the adversarial online learning problem and create a completely online algorithmic framework that has data dependent regret guarantees in both full expert feedback and bandit feedback settings. We study the expected performance of our algorithm against general comparators, which makes it applicable for a wide variety of problem scenarios. Our algorithm works from a universal prediction perspective and the performance measure used is the expected regret against arbitrary comparator sequences, which is the difference between our losses and a competing loss sequence. The competition class can be designed to include fixed arm selections, switching bandits, contextual bandits, periodic bandits or any other competition of interest. The sequences in the competition class are generally determined by the specific application at hand and should be designed accordingly. Our algorithm neither uses nor needs any preliminary information about the loss sequences and is completely online. Its performance bounds are data dependent, where any affine transform of the losses has no effect on the normalized regret.  
\end{abstract}

\section{Introduction}\label{sec:intro}

\subsection{Preliminaries}
In machine learning literature \cite{jordan2015machine,mohri2018foundations}, the field of online learning \cite{shalev2011online} is extensively researched across various domains such as game theory \cite{chang,tnnls1}, control theory \cite{sw2,sw4,tnnls3}, decision theory \cite{tnnls4,freund1997}, and computational learning theory \cite{comp1,comp2}. This field is heavily relied upon for its universal prediction perspective \cite{merhav}, which makes it a suitable choice for data and signal processing \cite{sw3,signal2,moon,signal1,sw5,gokcesu2018adaptive}, especially in sequential prediction and estimation problems \cite{gHierarchical,singer,singer2} like density estimation and anomaly detection \cite{gAnomaly,willems,coding1,coding2,gDensity}. Some of its most notable applications include multi-agent systems \cite{sw1,vanli,tekin2014distributed} and reinforcement learning problems \cite{bandit2,auerExp,audibert,auer,gBandit,bandit1,tekin2,exptrade,auerSelf,reinInt,mannucci2017safe}.

Specifically, the study of prediction with expert advice and online forecasting in adversarial scenarios has become an area of considerable interest, where the central aim is to minimize or maximize a certain loss or reward within a given environment \cite{cesabook}. In the prediction with expert advice problem \cite{signal2,signal1,singer,moon,singer2}, the focus is on a set of $M$ expert actions (e.g., algorithms) that can be employed for a specific task. At each round of the decision-making process, one of these expert actions is selected, and the corresponding loss or gain is received. The primary objective of research in this field is to design randomized online algorithms that can achieve low 'regret', which is defined as the difference between the expected loss and the loss of a strategy of expert selections \cite{littlestone1994, vovk1998}. This goal is essential in minimizing the effect of mistakes in decision-making. This problem is often encountered in real-world scenarios, such as in finance, where investors seek to maximize their returns by selecting the best investment options from a pool of available choices.

In some cases, the fundamental exploration-exploitation dilemma is encountered, which is most thoroughly studied in the multi-armed bandit problem \cite{cesa-bianchi}. The bandit setting represents a limited feedback version of the prediction with expert advice, where the main goal is to minimize or maximize a loss or reward by sequentially selecting actions while observing only the loss or reward of the chosen action \cite{cesabook}.

The concept of bandit feedback has gained significant attention due to its versatility in various applications \cite{bandit2,cesa-bianchi,audibert,auer,auerExp,banditTNN,zheng,gBandit}, such as online advertising \cite{li2010contextual}, recommender systems \cite{tekin2014distributed,tang2014ensemble,luo2015nonnegative}, clinical trials\cite{hardwick1991bandit}, and cognitive radio\cite{lai2008medium,gai2010learning}. We study the problem within the context of online learning, where observations are sequentially processed from an adversarial environment without any statistical assumptions on the loss sequences \cite{huang2011adversarial}. We utilize the competitive algorithm perspective to approach the problem, drawing upon existing research in the field of competitive analysis \cite{merhav,vural2019minimax,vovk,littlestone1994,neyshabouri2018asymptotically,vovk1998,gokcesu2020recursive}. Under this perspective, an algorithm's performance is measured against a competition class of selection strategies or comparators, with the goal of achieving a cumulative loss as close as possible to the cumulative losses of the arm selection sequences in the competition \cite{cesa2007}. The difference between the cumulative loss of the algorithm and the best selections on the same loss sequence is referred to as "regret" \cite{cesabook}. In the competitive algorithm framework, an algorithm does not require prior knowledge of the actions corresponding to each expert or arm, as they may even be separately running black-box algorithms that learn throughout time instead of specific actions. The only necessary prior knowledge is the number of actions, allowing an algorithm to associate a selection with each action based on observed sequential performance.

The problem of adversarial online learning, in which a player competes against the best fixed selection, is subject to a regret lower bound of $\Omega(\sqrt{T})$ in a game of $T$ rounds \cite{cesa-bianchi}. When competing against sequences in an arbitrary competition class, the regret lower bound implies a minimax bound of $\Omega(\sqrt{W_*T})$, where $W_*$ represents the complexity of the competition sequence. The complexity may depend on various factors, such as the number of switches \cite{cesa-bianchi,auer,auerExp,audibert,doubling_trick} or contextual regions \cite{willems1995context,sadakane2000implementing,willems1996context,csiszar2006context,dumont2014context,kozat2007universal,vanli2014comprehensive,ozkan2016online}, and may reflect a prior on the arm sequences \cite{comp2,gokcesu2020generalized}. State-of-the-art algorithms achieve expected minimax regret of $O(\sqrt{W_*T})$ when $W_*$ is known beforehand. They can also achieve a regret bound of $O(\sqrt{WT})$ when competing against a comparator class with complexities bounded by $W$. However, these algorithms are limited in that they assume losses are bounded, typically within the range of $[0,1]$. As a result of their lack of adaptivity, the regret bounds may not be fundamental when losses have unknown bounds.

\subsection{Comparisons}
In the field of prediction with expert advice, researchers have been searching for fundamental regret bounds that are invariant to translation and scaling for several decades. The exponentially weighted average forecaster \cite{littlestone1994,vovk1998} provides a basic regret bound against fixed competition that is dependent on a known universal loss bound and the number of rounds. 

The work in \cite{freund1997} showed that the exponentially weighted forecaster algorithm can achieve a first-order regret bound in one-sided games, where all losses have the same sign. 

A direct analysis by the work of \cite{allenberg2004} showed that the weighted majority algorithm achieves the first-order regret bound in signed games without the need for one-sidedness in the losses. These approaches are scale equivariant, but not translation invariant, since they require some information about the losses beforehand. 

The work of \cite{cesa2007} solved this shortcoming by creating second-order regret bounds for signed games that depend on the sum of squared losses and do not require any a priori knowledge. However, their results are not applicable in the bandit feedback setting. 

Similarly, \cite{ito2021parameter} derived data-dependent bounds for bandits, but they still require the losses to be universally bounded in the $[0,1]$ region. Additionally, these approaches mainly focus on fixed competition and have limited applicability to arbitrary comparators. 

The approach in \cite{gokcesu2020generalized} extends data-dependent regret bounds to a generalized framework that can compete against different choices of competitions in the problem of prediction with expert advice. However, their results are not applicable in the bandit setting due to the limited feedback. 

While the algorithms in \cite{gokcesu2021generalized,gokcesu2022second} circumvent this problem, their regret bounds are not bona-fide data-dependent.

\subsection{Contributions}
In light of the limitations of existing works, we propose an algorithm to address the problem of competing against arbitrary comparator sequences with data-dependent regret bounds. Our framework enables us to implement the desired competition class in a scalable and tractable manner. Our algorithm does not require explicit knowledge of the actions presented, and only relies on the prior knowledge of the number of available actions and the desired competition class. Our algorithm selects the corresponding action sequentially based solely on past performance and is parameter-free. Our regret bounds are completely data-dependent, i.e., depend upon the specific loss ranges in the realized loss sequence. 

\subsection{Organization}
The organization of the paper is as follows.
\begin{itemize}
	\item In \autoref{sec:problem}, we first describe the problem setting. 
	\item In \autoref{sec:method}, we detail the methodology and our algorithm. We provide the performance results and regret analysis in \autoref{sec:regret}. 
	\item In \autoref{sec:full}, we provide the implementation in full feedback setting.
	\item In \autoref{sec:bandit}, we provide the implementation for the bandit feedback setting.
\end{itemize}

\section{Problem Description}\label{sec:problem}
The prediction with expert advice, mixture-of-experts or multi-armed bandits problem is often formulated as a game between a decision-maker and an environment. The decision-maker must make a sequence of choices, and at each stage, they receive some feedback about the outcome of their decision. The central aim is to minimize or maximize a certain loss or reward within the given environment.

In this problem, we have a set of $M$ expert actions or algorithms that can be employed for a specific task such that $m\in\{1,\ldots,M\}$. At each round $t$ of the decision-making process, one of these expert actions is randomly selected and the corresponding loss is incurred. 

We denote our selection at time $t$ as $i_t$, which drawn from a probability distribution $q_t$ given by\begin{align}
	q_t\triangleq[q_{t,1},\ldots,q_{t,M}],\label{eq:qt}
\end{align}
i.e., $i_t\sim q_t$, and the probability of selection $i_t$ is $q_{t,i_t}$. Based upon our online selection $i_t$, we incur its corresponding loss $l_{t,i_t}$.

There are two variations of this problem: the full feedback setting, in which we observe the loss of each expert, i.e., $l_{t,m}, \forall m$ irrespective of the selection $i_t$ that we made; and the bandit feedback setting, where only the loss of the selected expert $i_t$, i.e., $l_{t,i_t}$ is observed. In both cases, our objective is to minimize regret, which is defined as the difference between our cumulative loss and that of a strategy chosen from the set of comparators. In a $T$ round game, let the selections made by the decision maker be cumulatively given by
\begin{align}
	I_T=[i_1,\ldots,i_T],
\end{align}
and its corresponding loss sequence be given by
\begin{align}
	L_{I_T}=[l_{1,{i_1}},\ldots,l_{T,i_T}].
\end{align}

To evaluate the decision-maker's performance, we define a selection sequence and its corresponding loss sequence for each comparator. Similarly, we define $S_T$ as the selection sequence of a comparator as
\begin{align}
	S_T=[s_1,\ldots,s_T].\label{eq:St}
\end{align}
such that each $s_t\in\{1,2,\ldots,M\}$ for all $t$. Hence, the corresponding loss sequence of the comparator $S_T$ is
\begin{align}
	L_{S_T}=[l_{1,{s_1}},\ldots,l_{T,s_T}].
\end{align}

We denote the cumulative loss up to round $T$ for each strategy $I_T$ and $S_T$ as $C_{I_T}$ and $C_{S_T}$, respectively. The regret of strategy $I_T$ relative to strategy $S_T$ is then defined as the difference between their cumulative losses. Let the expected regret against $S_T$ be $R_{S_T} =\E\left[C_{I_T} - C_{S_T}\right]$, where the expectation is over the selections $I_T$. The decision-maker's goal is to design an algorithm that achieves low expected regret against the comparator $S_T$. For data-dependent regret bounds, we want $R_{S_T}$ to be a function of all $l_{t,m}, 1\leq t\leq T,\forall m$ and the complexity $W_{S_T}$ of the comparator $S_T$.

\section{The Algorithm}\label{sec:method}

	\begin{algorithm}[!t]
		\caption{Generalized Selection Algorithm}\label{alg:framework}
		\small{\begin{algorithmic}[1]
				\For{$t=1 \to T$}
				\For{$m\in\{1,\ldots,M\}$}
				\State $$q_{t,m}=(1-\epsilon_{t})p_{t,m}+\epsilon_{t}\frac{1}{M}$$
				\EndFor
				\State Select $i_t\in\{1,\ldots,M\}$ with $q_t=[q_{t,1},\ldots,q_{t,M}]$
				\State Receive $\phi_{t}=[\phi_{t,1},\ldots,\phi_{t,M}]$
				\For{$\lambda_t\in\Omega_t$}
				\State $$z_{\lambda_t}=w_{\lambda_t}\exp(-\eta_{t-1}\phi_{t,\lambda_t(1)})$$
				\EndFor
				\For{$\lambda_{t+1}\in\Omega_{t+1}$}
				\State $$w_{\lambda_{t+1}}=\sum_{\lambda_{t}\in\Omega_t}\Tau(\lambda_{t+1}|\lambda_t)z_{\lambda_t}^{\frac{\eta_{t}}{\eta_{t-1}}}$$
				\EndFor
				\For{$m\in\{1,\ldots,M\}$}
				\State $$w_{t+1,m}={\sum_{\lambda_{t+1}(1)=m}w_{\lambda_{t+1}}}$$
				\EndFor
				\For{$m\in\{1,\ldots,M\}$}
				\State $$p_{t+1,m}=\frac{w_{t+1,m}}{\sum_{m'=1}^{M}w_{t+1,m'}}$$
				\EndFor
				\EndFor
		\end{algorithmic}}
	\end{algorithm}
	
	To construct our algorithm, we initiate by utilizing equivalence classes, which allow us to efficiently merge the arm selection sequences $S_t$ at each time step $t$.
	
	\subsection{Universal Prediction Perspective}
	Our proposed algorithm operates by assigning a weight $w_{S_t}$ to each of the possible arm selection sequences $S_t$ in an implicit manner. These weights are then utilized to determine the weights of each individual arm $m$ at time $t$ denoted as $w_{t,m}$. To obtain $w_{t,m}$, we collect all the sequences that suggest selecting arm $m$ at time $t$ and sum their corresponding weights, i.e.,
	\begin{align}
	w_{t,m}\triangleq\sum_{S_t(t:t)=m}^{} w_{S_t},\label{wmt}
	\end{align}
	where the vector $S_t(i:j)$ is defined as the vector containing elements from $i^{th}$ to $j^{th}$ positions of $S_t$, e.g., $S_t(t:t)=s_t$ refers to the selection of the sequence $S_t$ at time $t$. The merging of these sequences is performed in a way that naturally leads to the achievement of the best sequence's performance due to its universal perspective, as proposed in \cite{merhav}. We normalize the resulting values to obtain the algorithmic probabilities
	\begin{align}
	p_{t,m}=\frac{w_{t,m}}{\sum_{m'}w_{t,m'}},\label{pmt}
	\end{align}
	and selection probabilities $q_{t,m}$ are given by mixing $p_{t,m}$ with a uniform distribution as 
	\begin{align}
		q_{t,m}=(1-\epsilon_t)p_{t,m}+\epsilon_t\frac{1}{M},\label{qtm}
	\end{align}
	where $\epsilon_t$ is a time dependent parameter.
	
	\subsection{Equivalence Classes}
	It is worth noting that the computation of $p_{m,t}$ in \eqref{pmt} is directly influenced by $w_{S_t}$ because of \eqref{wmt}. These weights are calculated implicitly using equivalence classes, where we update specific sequence weights simultaneously.
	To establish the equivalence classes, we define a class parameter $\lambda_t$, which will be used to group the sequences $S_t$. Let
	\begin{align}
		\lambda_t=[m, \ldots],\label{lamt}
	\end{align}
	where the first parameter $\lambda_t(1)$ is arbitrarily determined as the selection $s_t$ at time $t$. The parameter $\lambda_t$ specifies which sequences are incorporated in its equivalence class, i.e., the equivalence class with parameter $\lambda_t$ comprises all sequences $S_t$ that behave consistently with $\lambda_t$. The parameters of $\lambda_t$ determine the number of equivalence classes and the sequences they represent. We define $\Omega_t$ as the vector space consisting of all $\lambda_t$ vectors as $\lambda_t\in\Omega_t, \enspace\forall\lambda_t.$ It is worth noting that $\Omega_t$ does not necessarily encompass all possible sequences at time $t$. Instead, it may encompasses only those sequences that are of interest to us and that we want to compete against. Moreover, we define $\Lambda_t$ as the parameter sequence for any arbitrary sequence up to time $t$ as $\Lambda_t\triangleq\{\lambda_1,\ldots,\lambda_t\}.$ Naturally, each sequence $S_t$ corresponds to only one $\Lambda_t$. We define $w_{\lambda_t}$ as the weight of the set of equivalence class parameters $\lambda_t$ at time $t$. This weight is the sum of the implicit weights of the sequences whose behavior matches $\lambda_t$, i.e.,
	\begin{align}
	w_{\lambda_t}=\sum_{F_\lambda(S_t)=\lambda_t}^{}w_{S_t},\label{wlt}
	\end{align}
	where $F_\lambda(\cdot)$ maps sequences $S_t$ to class parameters $\lambda_t$, which is used to transform the definition of $w_{t,m}$ in equation \eqref{wmt} as
	\begin{align}
	w_{t,m}=\sum_{\lambda_t(1)=m}^{} w_{\lambda_t}.\label{wmt2}
	\end{align}

	Our weight update strategy for the parameter sequence $w_{\lambda_t}$ follows a two-step approach. Initially, we introduce an intermediate variable $z_{\lambda_t}$ that incorporates an exponential performance update, similar to the exponential weighting algorithm \cite{cesabook}. The variable $z_{\lambda_t}$ is defined as
	\begin{align}
		z_{\lambda_t}\triangleq w_{\lambda_t}e^{-\eta_{t-1}\phi_{t,\lambda_t(1)}},\label{zlt}
	\end{align}
	where $\phi_{t,m}$ measures the selection performance, which will be explained in the subsequent section.
	Secondly, we construct a probability sharing network between the equivalence classes, which implicitly represents and assigns a weight to each individual sequence $S_t$ at time $t$. We achieve this by computing the sum of products between the transition weights $\Tau(\lambda_{t+1}|\lambda_t)$ and the powered $z_{\lambda_t}$ values, where the power normalization on $z_{\lambda_t}$ is necessary for adaptive learning rates. The equation that governs this probability sharing network is 
	\begin{align}
		w_{\lambda_{t+1}}=\sum_{\lambda_t\in\Omega_t}\Tau(\lambda_{t+1}|\lambda_t)z_{\lambda_t}^{\frac{\eta_t}{\eta_{t-1}}},\label{wlt+}
	\end{align}
	where $\Tau(\lambda_{t+1}|\lambda_t)$ is the transition weight from class parameters $\lambda_t$ to $\lambda_{t+1}$ such that $\sum_{\lambda_{t+1}\in\Omega_{t+1}}\Tau(\lambda_{t+1}|\lambda_t)=1$. It is worth noting that this weight update approach assigns a weight to every individual sequence $S_t$ implicitly. We provide a summary of our proposed method in Algorithm \ref{alg:framework}.
		\subsection{Example Implementations}\label{sec:example}
	All the algorithmic parameters $\eta_t$, $\epsilon_t$, $\phi_{t,m}$ will be gradually designed in the next section to provide the aimed regret bounds. We note that given the problem setting, the class parameters $\lambda_t$ can include auxiliary information that either defines an arm selection rule for hyper-experts, or groups together class transition updates. While our algorithm is generic and applicable for various comparator classes, we provide some example designs for $\lambda_t$ and $\Tau(\lambda_{t+1}|\lambda_{t})$ to increase comprehension.
	
	\begin{itemize}
		\item For fixed bandits \cite{auerExp}, we have $\lambda_t=\{m\}$, $\Tau(\lambda_1)=1/M$, $$\Tau(\lambda_{t+1}|\lambda_t)=\begin{cases}
			1, &\lambda_{t+1}(1)=\lambda_t(1)\\
			0, &\lambda_{t+1}(1)\neq\lambda_t(1)
		\end{cases}.
		$$
		The computational complexity of this design is $O(M)$.
		
		\item For switching bandits \cite{gBandit}, where $\lambda_t=\{m,\tau\}$, we can utilize the following transitions:
		
		\begin{align*}
			\Tau(\lambda_{t+1}|\lambda_t)
			=\begin{cases}
				\displaystyle1-\frac{1}{\lambda_{t}(2)+1},& 
				{}^{\lambda_{t+1}(1)=\lambda_{t}(1)}_{\lambda_{t+1}(2)=\lambda_{t}(2)+1}
				\\
				\displaystyle\frac{1}{M-1}{\frac{1}{\lambda_{t}(2)+1}},& 
				{}^{\lambda_{t+1}(1)\neq\lambda_{t}(1)}_{\lambda_{t+1}(2)=1}\\
				0,& otherwise
			\end{cases}.
		\end{align*}
		The computational complexity of this design is $O(MT)$, we can be decreased to $O(M)$ by making the transition weights only dependent on the current round $t$.
		
		\item For contextual bandits \cite{neyshabouri2018asymptotically}, instead of keeping track of the arm played $m$, we need to keep a mapping $\mathcal{M}\colon\{1,2,\ldots,N\}\rightarrow\{1,2,\ldots,M\}$ between the context and the arms. Here, $c_t\in\{1,2,\ldots,N\}$ is the ordered context received at time $t$. Here, we set the class parameters as $\lambda_t=\{\mathcal{M}, c\}$, where $c$ is the number of context regions (each region refers to consecutive context with same mapping)
		\begin{align*}
			\Tau(\lambda_{t+1}|\lambda_t)=\begin{cases}
				\displaystyle1-\frac{1}{t}, &\scriptstyle\lambda_{t+1}(1)=\lambda_t(1)\\
				\displaystyle\frac{1}{t}\left(\frac{1}{2NM}\right)^{\lambda_{t+1}(2)}, &\scriptstyle\lambda_{t+1}(1)\neq\lambda_t(1)
			\end{cases}.
		\end{align*}
		With the efficient implementation in \cite{willems1995context}, complexity $O(M\log N)$ can be achieved.
		
		\item For periodic bandits \cite{oh2019periodic}, we need to keep a mapping $\mathcal{G}_\tau\colon \{1,2,\ldots,\tau\}\rightarrow\{1,2,\ldots,M\}$, for a relevant period length $\tau\in$. Here, we set the class parameters as $\lambda_t=\{\mathcal{G}_\tau, \tau\}$, and
		\begin{align*}
			\Tau(\lambda_{t+1}|\lambda_t)=\begin{cases}
				1-\frac{1}{t}, &\scriptstyle\lambda_t(1)=\lambda_{t+1}(1)\\
				\frac{1}{2t}(2M)^{\lambda_{t}(2)-\lambda_{t+1}(2)}, &\scriptstyle\lambda_t(1)\subset\lambda_{t+1}(1)\\
				\frac{1}{2t}(2M)^{-\lambda_{t+1}(2)}, &\scriptstyle\lambda_t(1)\not\subseteq\lambda_{t+1}(1)
			\end{cases}.
		\end{align*}
		When we compete against sequences with at most $\tau_B$ period, $O(M\tau_B)$ complexity can be achieved with an efficient implementation.
		
	\end{itemize}

\section{Expected Cumulative Regret Analysis}\label{sec:regret}
	In this section, we study the performance of our algorithm. We first provide a summary of some important notations and definitions, which will be heavily used. Then, we will incrementally study the relevant regret bounds.
	\subsection{Notations and Definitions}\label{secsec:not}
	\begin{enumerate}
		\item $q_{t,m}$ is the probability of selecting $m$ at $t$ as in \eqref{qtm}.
		\item $\E_{f_{t,m}}[x_{t,m}]$ is the convex sum of $x_{t,m}$ with the coefficients $f_{t,m}$, i.e., $\sum_{m=1}^{M}f_{t,m}x_{t,m}$.
		\item $\E_{f_{t}}[x]$ is the expectation of $x$ when $i_t$ is drawn from $f_t$.
		\item $\E[x]$ is the expectation of $x$ when $i_t$ is drawn from $q_{t,m}$.
		\item $\eta_t$ is the learning rate used in \eqref{zlt}.
		\item $\phi_{t,m}$ is the performance metric used in \eqref{zlt}.
		\item $d_t\triangleq \enspace\max_m\phi_{t,m}-\min_m\phi_{t,m}$.
		\item $v_t\triangleq \enspace\E_{p_{t,m}}\phi_{t,m}^2$.
		\item $D_t\triangleq\max_{1\leq t' \leq t}d_t,$.
		\item $V_t\triangleq\sum_{t'=1}^t v_t$. 
		\item $\lambda_t$ is an equivalence class parameter at time $t$ as in \eqref{lamt}.
		\item $\Omega_t$ is the set of all $\lambda_t$ at time $t$.
		\item $\Lambda_T\triangleq\{\lambda_t\}_{t=1}^T$.
		\item $z_{\lambda_{t}}$ is as in \eqref{zlt}.
		\item $\Tau(\cdot|\cdot)$ is the transition weight used in \eqref{wlt+}.
		\item $\Tau(\{\lambda_t\}_{t=1}^T)\triangleq\prod_{t=1}^T\Tau(\lambda_t|\lambda_{t-1})$.		
		\item $W(\Lambda_T)\triangleq \log(\max_{1\leq t\leq T}|\Omega_{t-1}|)-\log(\Tau(\Lambda_T))$, which corresponds to the complexity of a competition.
		\item $e$ is the Euler coefficient.
		\item $\log(\cdot)$ is the natural logarithm.
\end{enumerate}

	\subsection{Useful Lemmas}
	We start our regret analyses with some useful lemmas. We begin by making an assumption on the design choices of the learning rate $\eta_t$ and the performance metric $\phi_{t,m}$.
\begin{assumption}\label{ass:-eta_tphi_tm}
	Let $-\eta_t\phi_{t,m}\leq 1, \forall t,m$.
\end{assumption}

To derive the regret bounds of our framework, we first determine a term of interest 
\begin{align}
	\frac{1}{\eta_t}\log\E_{p_{t,m}}[e^{-\eta_t\phi_{t,m}}],\label{Eexp}
\end{align}
and use it to derive some useful Lemmas.
\begin{lemma}\label{lem:lB}
	For any probability simplex $p_{t,m}$, we have the following inequality
	\begin{align}
		\frac{1}{\eta_t}\log\E_{p_{t,m}}[e^{-\eta_t\phi_{t,m}}]\leq-\E_{p_{t,m}}\phi_{t,m}+\eta_t\E_{p_{t,m}}\phi_{t,m}^2,\nonumber
	\end{align}
	when we have \autoref{ass:-eta_tphi_tm}.
	\begin{proof}
		The proof uses the inequality $e^x\leq1+x+x^2$, when $x\leq1$ (which is comes from Taylor series \cite{handbook}).
	\end{proof}		
\end{lemma}
The preceding result puts an upper bound to our term of interest in \eqref{Eexp}. Similarly, we also have the following result, which is a lower bound to that same term.

\begin{lemma}\label{lem:uB1}
	For any probability simplex $p_{t,m}$, we have the following inequality
	\begin{align}
		\frac{1}{\eta_{t}}\log\E_{p_{t,m}}[e^{-\eta_{t}\phi_{t,m}}]\geq
		\frac{1}{\eta_{t-1}}&\log\E_{p_{t,m}}[e^{-\eta_{t-1}\phi_{t,m}}]\nonumber\\
		&-\left|1-\frac{\eta_t}{\eta_{t-1}}\right|d_t,\nonumber
	\end{align}
	where the operation $|\cdot|$ gives the absolute value. 
	\begin{proof}
		First of all, we have
		\begin{align}
			\frac{1}{\eta_{t-1}}\log&\E_{p_{t,m}}[e^{-\eta_{t-1}\phi_{t,m}}]\nonumber\\	&=\frac{1}{\eta_{t-1}}\log\E_{p_{t,m}}[e^{-\eta_{t}\phi_{t,m}+(\eta_t-\eta_{t-1})\phi_{t,m}}]\nonumber\\
			&\leq\frac{1}{\eta_{t-1}}\log\E_{p_{t,m}}[e^{-\eta_{t}\phi_{t,m}+(\eta_t-\eta_{t-1})a_t}]\nonumber\\
			&\leq\frac{1}{\eta_{t-1}}\log\E_{p_{t,m}}[e^{-\eta_{t}\phi_{t,m}}]+\left(\frac{\eta_t}{\eta_{t-1}}-1\right)a_t,\label{eq:2}
		\end{align}
		where $a_t$ is either minimum or maximum of $\phi_{t,m}$ over $m$ depending on whether or not $\eta_{t-1}$ is greater than $\eta_{t}$, i.e.,
		\begin{align}
			a_t=\left\{
			\begin{array}{ll}
				\min_m\phi_{t,m} & \eta_t\leq\eta_{t-1}\\
				\max_m\phi_{t,m} & \eta_t\geq\eta_{t-1}
			\end{array}.
			\right.
		\end{align}
		Secondly, we also have
		\begin{align}
			\frac{1}{\eta_{t-1}}\log\E_{p_{t,m}}[e^{-\eta_{t}\phi_{t,m}}]-&\frac{1}{\eta_{t}}\log\E_{p_{t,m}}[e^{-\eta_{t}\phi_{t,m}}]\nonumber\\
			&\leq\left(\frac{1}{\eta_{t-1}}-\frac{1}{\eta_{t}}\right)\log\E_{p_{t,m}}[e^{-\eta_{t}b_t}]\nonumber\\
			&\leq-\left(\frac{\eta_t}{\eta_{t-1}}-1\right)b_t,\label{eq:3}
		\end{align}
		where $b_t$ is either minimum or maximum of $\phi_{t,m}$ over $m$ depending on whether or not $\eta_{t}$ is greater than $\eta_{t-1}$, i.e.,
		\begin{align}
			b_t=\left\{
			\begin{array}{ll}
				\max_m\phi_{t,m} & \eta_t\leq\eta_{t-1}\\
				\min_m\phi_{t,m} & \eta_t\geq\eta_{t-1}
			\end{array}.
			\right.
		\end{align}
		Combining \eqref{eq:2} and \eqref{eq:3}, we get
		\begin{align}
			\frac{1}{\eta_{t-1}}\log\E_{p_{t,m}}[e^{-\eta_{t-1}\phi_{t,m}}]&-\frac{1}{\eta_{t}}\log\E_{p_{t,m}}[e^{-\eta_{t}\phi_{t,m}}]\nonumber\\
			&\leq\left(\frac{\eta_t}{\eta_{t-1}}-1\right)(a_t-b_t)\nonumber\\
			&\leq\left|\frac{\eta_t}{\eta_{t-1}}-1\right|d_t,\label{eq:4}
		\end{align}
		where $d_t\triangleq(\max_m\phi_{t,m}-\min_m\phi_{t,m})$. Moreover, since
		\begin{align}
			-\frac{1}{\eta_t}\log\E_{p_{t,m}}[e^{-\eta_t\phi_{t,m}}]=&-\frac{1}{\eta_{t-1}}\log\E_{p_{t,m}}[e^{-\eta_{t-1}\phi_{t,m}}]\nonumber\\
			&+\frac{1}{\eta_{t-1}}\log\E_{p_{t,m}}[e^{-\eta_{t-1}\phi_{t,m}}]\nonumber\\
			&-\frac{1}{\eta_t}\log\E_{p_{t,m}}[e^{-\eta_t\phi_{t,m}}]\label{eq:1},
		\end{align}
		putting \eqref{eq:4} into \eqref{eq:1} concludes the proof.
	\end{proof}
\end{lemma}
Note that \autoref{lem:uB1} provides only a partial bound. To further bound the term, we have the following assumption.

\begin{assumption}\label{ass:etatnoninc}
	Let $n_t$ be nonincreasing with $t$, i.e., $n_{t+1}\leq n_{t}, \forall t$.
\end{assumption}
We continue from \autoref{lem:uB1} to bound the excess term on the right hand side. 
\begin{lemma}\label{lem:uB22}
	When using \autoref{alg:framework}, we have
	\begin{align*}
		\log\E_{p_{t,m}}[e^{-\eta_{t-1}\phi_{t,m}}]&=\log\left(\frac{\sum_{\lambda_t\in\Omega_t}z_{\lambda_t}}{\sum_{\lambda_{t-1}\in\Omega_{t-1}}z_{\lambda_{t-1}}^{\frac{\eta_{t-1}}{\eta_{t-2}}}}\right)
	\end{align*}
	when we have \autoref{ass:etatnoninc}.
\end{lemma}
\begin{proof}
	From \eqref{pmt}, \eqref{wlt}, \eqref{zlt} and \eqref{wlt+}; we have
	\begin{align}
		-\frac{1}{\eta_{t-1}}\log&\left(\E_{p_{t,m}}[e^{-\eta_{t-1}\phi_{t,m}}]\right)\nonumber\\
		&=-\frac{1}{\eta_{t-1}}\log\left(\frac{\sum_{m}w_{t,m}e^{-\eta_{t-1}\phi_{t,m}}}{\sum_{m'}w_{t,m'}}\right)\label{eq:pr1}\\
		&=-\frac{1}{\eta_{t-1}}\log\left(\frac{\sum_{\lambda_t\in\Omega_t}w_{\lambda_t}e^{-\eta_{t-1}\phi_{t,\lambda_t(1)}}}{\sum_{\lambda_t\in\Omega_t}w_{\lambda_t}}\right)\label{eq:pr2}\\
		&=-\frac{1}{\eta_{t-1}}\log\left(\frac{\sum_{\lambda_t\in\Omega_t}z_{\lambda_t}}{\sum_{\lambda_t\in\Omega_t}w_{\lambda_t}}\right),\label{eq:pr3}\\
		&=-\frac{1}{\eta_{t-1}}\log\left(\frac{\sum_{\lambda_t\in\Omega_t}z_{\lambda_t}}{\sum_{\lambda_{t-1}\in\Omega_{t-1}}z_{\lambda_{t-1}}^{\frac{\eta_{t-1}}{\eta_{t-2}}}}\right),\label{eq:5}
	\end{align}
	where \eqref{eq:pr1}, \eqref{eq:pr2}, \eqref{eq:pr3} and \eqref{eq:5} use results from \eqref{pmt}, \eqref{wlt}, \eqref{zlt} and \eqref{wlt+} respectively.
\end{proof}
	
	\begin{lemma}\label{lem:uB2}
		When using \autoref{alg:framework}, we have
		\begin{align}
			\frac{1}{\eta_{t-1}}\log\E_{p_{t,m}}[e^{-\eta_{t-1}\phi_{t,m}}]&\geq\frac{1}{\eta_{t-1}}\log\left({\sum_{\lambda_t\in\Omega_t}z_{\lambda_t}}\right)\nonumber\\
			&-\frac{1}{\eta_{t-2}}\log\left({\sum_{\lambda_{t-1}\in\Omega_{t-1}}z_{\lambda_{t-1}}}\right)\nonumber\\
			&-\left(\frac{1}{\eta_{t-1}}-\frac{1}{\eta_{t-2}}\right)\log(|\Omega_{t-1}|),\nonumber
		\end{align}
		when we have \autoref{ass:etatnoninc}.
	\begin{proof}
		We start by bounding the denominator in the logarithm of \autoref{lem:uB22} by 
		\begin{align}
			&\frac{1}{\eta_{t-1}}\log\left({\sum_{\lambda_{t-1}\in\Omega_{t-1}}z_{\lambda_{t-1}}^{\frac{\eta_{t-1}}{\eta_{t-2}}}}\right)\nonumber\\
			&\leq\frac{1}{\eta_{t-1}}\log\left({\sum_{\lambda_{t-1}}\frac{1}{|\Omega_{t-1}|}z_{\lambda_{t-1}}^{\frac{\eta_{t-1}}{\eta_{t-2}}}}\right)+\frac{\log(|\Omega_{t-1}|)}{\eta_{t-1}},\nonumber\\
			&\leq\frac{1}{\eta_{t-2}}\frac{\eta_{t-2}}{\eta_{t-1}}\log\left({\sum_{\lambda_{t-1}}\frac{1}{|\Omega_{t-1}|}z_{\lambda_{t-1}}^{\frac{\eta_{t-1}}{\eta_{t-2}}}}\right)+\frac{\log(|\Omega_{t-1}|)}{\eta_{t-1}},\nonumber
		\end{align}
		where the set $\Omega_{t-1}$ is omitted over the summations after the first line for space considerations. 
		
		From Jensen's Inequality, we get
		\begin{align}
			&\frac{1}{\eta_{t-1}}\log\left({\sum_{\lambda_{t-1}\in\Omega_{t-1}}z_{\lambda_{t-1}}^{\frac{\eta_{t-1}}{\eta_{t-2}}}}\right)\nonumber\\
			&\leq\frac{1}{\eta_{t-2}}\log\left({\sum_{\lambda_{t-1}}\frac{1}{|\Omega_{t-1}|}z_{\lambda_{t-1}}}\right)+\frac{\log(|\Omega_{t-1}|)}{\eta_{t-1}},\label{eq:pr6.1}\\
			&\leq\frac{1}{\eta_{t-2}}\log\left({\sum_{\lambda_{t-1}}z_{\lambda_{t-1}}}\right)+\left(\frac{1}{\eta_{t-1}}-\frac{1}{\eta_{t-2}}\right)\log(|\Omega_{t-1}|),\label{eq:6}
		\end{align}
		Putting \eqref{eq:6} into \autoref{lem:uB22}, we get
		\begin{align}
			\frac{1}{\eta_{t-1}}\log\E_{p_{t,m}}[e^{-\eta_{t-1}\phi_{t,m}}]&\geq\frac{1}{\eta_{t-1}}\log\left({\sum_{\lambda_t\in\Omega_t}z_{\lambda_t}}\right)\nonumber\\
			&-\frac{1}{\eta_{t-2}}\log\left({\sum_{\lambda_{t-1}\in\Omega_{t-1}}z_{\lambda_{t-1}}}\right)\nonumber\\
			&-\left(\frac{1}{\eta_{t-1}}-\frac{1}{\eta_{t-2}}\right)\log(|\Omega_{t-1}|),\label{eq:7}
		\end{align}
		which concludes the proof.
	\end{proof}
	\end{lemma}

In \autoref{lem:uB2}, we have succeeded in bounding the individual terms (at time $t$). However, our goal is to bound their summation (from $t=1$ to $T$), which requires the following result.
\begin{lemma}\label{lem:uB3}
	When using \autoref{alg:framework}, we have
	\begin{align*}
		\frac{1}{\eta_{T-1}}\log(z_{\lambda_T})\geq&\sum_{t=1}^T\frac{1}{\eta_{t-1}}\log(\Tau(\lambda_t|\lambda_{t-1}))-\sum_{t=1}^T\phi_{t,\lambda_{t}(1)},
	\end{align*}
	for any sequence of classes $\{\lambda_{t}\}_{t=1}^T$ when $|\Omega_0|=1$.
	\begin{proof}
		By definition in \eqref{zlt}, we have
		\begin{align}
			-\log(z_{\lambda_t})=\eta_{t-1}\phi_{t,\lambda_{t}(1)}-\log(w_{\lambda_{t}}),\label{eq:logzlt}
		\end{align}
		and from \eqref{wlt+}, we have	
		\begin{align}
			-\log(w_{t,\lambda_t})\leq-\log(\Tau(\lambda_{t}|\lambda_{t-1}))-\frac{\eta_{t-1}}{\eta_{t-2}}\log(z_{\lambda_{t-1}}).\label{eq:logwlt+}
		\end{align}
		Combining \eqref{eq:logzlt} and \eqref{eq:logwlt+}, we get
		\begin{align}
			-\frac{1}{\eta_{t-1}}\log(z_{\lambda_t})=&\enspace\phi_{t,\lambda_{t}(1)}-\frac{1}{\eta_{t-1}}\log w_{\lambda_t}\nonumber\\
			\leq&\enspace\phi_{t,\lambda_{t}(1)}-\frac{1}{\eta_{t-1}}\log(\Tau(\lambda_t|\lambda_{t-1}))\nonumber\\
			&-\frac{1}{\eta_{t-2}}\log(z_{\lambda_{t-1}}).\label{eq:logzlt+}
		\end{align}
		From the telescoping relation in \eqref{eq:logzlt+}, we get
		\begin{align}
			-\frac{1}{\eta_{T-1}}\log(z_{\lambda_T})\leq&\sum_{t=1}^T\phi_{t,\lambda_{t}(1)}-\sum_{t=1}^T\frac{1}{\eta_{t-1}}\log(\Tau(\lambda_t|\lambda_{t-1})),
		\end{align}
		since $z_{0,m}=1$ and concludes the proof.
	\end{proof}
\end{lemma}
Now, we can combine \autoref{lem:uB1}, \autoref{lem:uB2} and \autoref{lem:uB3} in the following to provide a lower bound to the summation of interest.

\begin{lemma}\label{lem:uB4}
	When using \autoref{alg:framework}, we have
	\begin{align*}
		\sum_{t=1}^T\frac{1}{\eta_{t}}\log\left(\E_{p_{t,m}}[e^{-\eta_t\phi_{t,m}}]\right)\geq&-\sum_{t=1}^T\phi_{t,\lambda_t(1)}
		\\&-\sum_{t=1}^{T}\left(1-\frac{\eta_t}{\eta_{t-1}}\right)d_t\nonumber\\
		&+\sum_{t=1}^T\frac{1}{\eta_{t-1}}\log\left(\Tau(\lambda_{t}|\lambda_{t-1})\right)\nonumber\\
		&-\frac{1}{\eta_{T-1}}\log\left(\max_{1\leq t\leq T}|\Omega_{t-1}|\right),
	\end{align*}
	when we have \autoref{ass:etatnoninc}.
	\begin{proof}
		We sum \eqref{eq:7} from $t=1$ to $T$, and get
		\begin{align}
			\sum_{t=1}^{T}\frac{1}{\eta_{t-1}}&\log\E_{p_{t,m}}[e^{-\eta_{t-1}\phi_{t,m}}]\nonumber\\
			\geq&\enspace\frac{1}{\eta_{T-1}}\log\left(\sum_{\lambda_{T}\in\Omega_{T}}z_{\lambda_{T}}\right)-\frac{1}{\eta_{-1}}\log\left(\sum_{\lambda_{0}\in\Omega_{0}}z_{\lambda_{0}}\right)\nonumber\\
			&-\sum_{t=1}^{T}\left(\frac{1}{\eta_{t-1}}-\frac{1}{\eta_{t-2}}\right)\log(|\Omega_{t-1}|),
		\end{align}
		where $\eta_{-1}$ and $\eta_{0}$ can be arbitrarily chosen as $\eta_{1}$ and $|\lambda_0|$ as $1$, $z_{\lambda_{0}}=1$.
		Then, using \autoref{lem:uB1}, \autoref{lem:uB2} and \autoref{lem:uB3}, we get
		\begin{align}
			\sum_{t=1}^T\frac{1}{\eta_{t}}\log\left(\E_{p_{t,m}}[e^{-\eta_t\phi_{\lambda_t}}]\right)\geq&\sum_{t=1}^{T}(\frac{1}{\eta_{t-2}}-\frac{1}{\eta_{t-1}})\log(|\Omega_{t-1}|)\nonumber
			\\&-\sum_{t=1}^T\phi_{t,\lambda_t(1)}\nonumber
			\\&-\sum_{t=1}^{T}\left(1-\frac{\eta_t}{\eta_{t-1}}\right)d_t\nonumber\\
			&+\sum_{t=1}^T\frac{1}{\eta_{t-1}}\log\left(\Tau(\lambda_{t}|\lambda_{t-1})\right).\label{eq:8}
		\end{align}
		Since $\eta_t\leq\eta_{t-1}$ and $|\Omega_t|\geq1$, \eqref{eq:8} becomes
		\begin{align}
			\sum_{t=1}^T\frac{1}{\eta_{t}}\log\left(\E_{p_{t,m}}[e^{-\eta_t\phi_{m}}]\right)\geq&-\sum_{t=1}^T\phi_{t,\lambda_t(1)}\nonumber
			\\&-\sum_{t=1}^{T}\left(1-\frac{\eta_t}{\eta_{t-1}}\right)d_t\nonumber\\
			&+\sum_{t=1}^T\frac{1}{\eta_{t-1}}\log\left(\Tau(\lambda_{t}|\lambda_{t-1})\right)\nonumber\\
			&-\frac{1}{\eta_{T-1}}\log(\max_{1\leq t\leq T}|\Omega_{t-1}|),
		\end{align}
		which concludes the proof.
	\end{proof}
\end{lemma}
With \autoref{lem:uB4}, we now have a lower bound to our summation of interest, which is the summation of our term of interest in \eqref{Eexp} from $t=1$ to $T$. 
	
\subsection{Expected Cumulative Regret Bound}
We combine \autoref{lem:lB} and \autoref{lem:uB4}, which are upper and lower bounds to our summation of interest, respectively, to get the following result.
\begin{lemma}\label{thm:bound1}
	When using \autoref{alg:framework}, we have
	\begin{align*}
		\sum_{t=1}^T\left(\E_{p_{t,m}}\phi_{t,m}-\phi_{t,\lambda_t(1)}\right)\leq& \sum_{t=1}^T\eta_t\E_{p_{t,m}}\phi_{t,m}^2\\
		&+\frac{\log(\max_{1\leq t\leq T}|\Omega_{t-1}|)}{\eta_{T-1}}\\
		&-\frac{1}{\eta_{T-1}}\log(\Tau(\Lambda_T))\\
		&+\sum_{t=1}^T\left(1-\frac{\eta_t}{\eta_{t-1}}\right)d_t,
	\end{align*}
	where $\Tau(\Lambda_T)=\Tau(\{\lambda_t\}_{t=1}^T)$; and we have \autoref{ass:-eta_tphi_tm} and \autoref{ass:etatnoninc}.
	\begin{proof}
		We combine the results of \autoref{lem:lB} and \autoref{lem:uB4} to reach
		\begin{align}
			\sum_{t=1}^T\left(\E_{p_{t,m}}\phi_{t,m}-\phi_{t,m_t}\right)\leq& \sum_{t=1}^T\eta_t\E_{p_{t,m}}\phi_{t,m}^2\nonumber\\
			&+\frac{\log(\max_{1\leq t\leq T}|\Omega_{t-1}|)}{\eta_{T-1}}\nonumber\\
			&-\sum_{t=1}^T\frac{1}{\eta_{t-1}}\log(\Tau(\lambda_t|\lambda_{t-1}))\nonumber\\
			&+\sum_{t=1}^T\left(1-\frac{\eta_t}{\eta_{t-1}}\right)d_t.
		\end{align}
		Since $\eta_t$ is nonincreasing with time $t$ and $\Tau(\lambda_t|\lambda_{t-1})\leq 1$, we get
		\begin{align}
			\sum_{t=1}^T\left(\E_{p_{t,m}}\phi_{t,m}-\phi_{t,m_t}\right)\leq& \sum_{t=1}^T\eta_t\E_{p_{t,m}}\phi_{t,m}^2\nonumber\\
			&+\frac{\log(\max_{1\leq t\leq T}|\Omega_{t-1}|)}{\eta_{T-1}}\nonumber\\
			&-\frac{1}{\eta_{T-1}}\log(\Tau(\Lambda_T))\nonumber\\
			&+\sum_{t=1}^T\left(1-\frac{\eta_t}{\eta_{t-1}}\right)d_t,
		\end{align}
		where $\Tau(\Lambda_T)\triangleq\prod_{t=1}^T\Tau(\lambda_t|\lambda_{t-1})$, which concludes the proof.
	\end{proof}
\end{lemma}

The result in \autoref{thm:bound1} provides us with a performance bound with respect to the algorithmic probabilities $p_{t,m}$.
To simplify the result in \autoref{thm:bound1}, we make the following assumption.
\begin{assumption}\label{ass:etatdt}
	Let $\eta_td_t\leq A_T, \forall t\leq T$.
\end{assumption}

Using \autoref{ass:etatdt} and the relevant notation, we have the following simplified result.

\begin{lemma}\label{thm:bound11}
	When using \autoref{alg:framework}, we have
	\begin{align*}
		\sum_{t=1}^T\left(\E_{p_{t,m}}\phi_{t,m}-\phi_{t,\lambda_t(1)}\right)\leq& \sum_{t=1}^T\eta_tv_t+\frac{W(\Lambda_T)+A_T}{\eta_T},
	\end{align*}
	where $v_t\triangleq \enspace\E_{p_{t,m}}\phi_{t,m}^2$, $W(\Lambda_T)\triangleq \log(\max_{1\leq t\leq T}|\Omega_{t-1}|)-\log(\Tau(\Lambda_T))$; and we have \autoref{ass:-eta_tphi_tm}, \autoref{ass:etatnoninc} and \autoref{ass:etatdt}.
	\begin{proof}
		From \autoref{ass:etatnoninc} and \autoref{ass:etatdt}, we have
		\begin{align}
			\sum_{t=1}^{T}\left(1-\frac{\eta_t}{\eta_{t-1}}\right)d_t=&\sum_{t=1}^{T}\left(\frac{1}{\eta_t}-\frac{1}{\eta_{t-1}}\right)\eta_td_t
			\\\leq&\sum_{t=1}^{T}\left(\frac{1}{\eta_t}-\frac{1}{\eta_{t-1}}\right)A_T
			\\\leq&\frac{A_T}{\eta_T}.\label{eq:ATnT}
		\end{align}
		Putting \eqref{eq:ATnT} in \autoref{thm:bound1} concludes the proof.
	\end{proof}
\end{lemma}

\begin{assumption} \label{ass:phitm}
	Let $\phi_{t,m}$ be such that
	\begin{align}
		\E_{i_t\sim q_t}[\phi_{t,m}]=l_{t,m}-\mu_{t,m}.
	\end{align}
\end{assumption}

Combining \autoref{thm:bound11} and \autoref{ass:phitm}, we have the following.
\begin{lemma}\label{thm:bound12}
	When using \autoref{alg:framework}, we have
	\begin{multline}
		\sum_{t=1}^T\left(\E_{p_{t,m}}[l_{t,m}-\mu_{t,m}]-l_{t,\lambda_t(1)}+\mu_{t,\lambda_t(1)}\right)
		\\\leq \E\left[\sum_{t=1}^T\eta_tv_t\right]+\E\left[\frac{W(\Lambda_T)+A_T}{\eta_T}\right],
	\end{multline}
	where $v_t\triangleq \enspace\E_{p_{t,m}}\phi_{t,m}^2$, $W(\Lambda_T)\triangleq \log(\max_{1\leq t\leq T}|\Omega_{t-1}|)-\log(\Tau(\Lambda_T))$; and we have \autoref{ass:-eta_tphi_tm}, \autoref{ass:etatnoninc}, \autoref{ass:etatdt} and \autoref{ass:phitm}.
	\begin{proof}
		Result follows from \autoref{ass:phitm} and the expectation of \autoref{thm:bound11}.
	\end{proof}
\end{lemma}

To relate the expected losses over probabilities $p_{t,m}$ to $q_{t,m}$ we have the following result. 
\begin{proposition}\label{thm:ptmqtm}
	We have
	\begin{align*}
		\E_{p_{t,m}}[l_{t,m}]\geq&\E_{q_{t,m}}[l_{t,m}]-\epsilon_t\delta_t,
	\end{align*}
	where $\delta_t\triangleq -\E_{p_{t,m}}[l_{t,m}]+\E_u[l_{t,m}]$.
	\begin{proof}
		From \eqref{qtm}, we have
		\begin{align}
			\E_{q_{t,m}}[l_{t,m}]=&(1-\epsilon_t)\E_{p_{t,m}}[l_{t,m}]+\epsilon_{t}\E_u[l_{t,m}].
		\end{align}
		Hence, we get
		\begin{align}
			\E_{p_{t,m}}[l_{t,m}]=&\E_{q_{t,m}}[l_{t,m}]+\epsilon_t(\E_{p_{t,m}}[l_{t,m}]-\E_u[l_{t,m}])
			\\\geq&\E_{q_{t,m}}[l_{t,m}]-\epsilon_t\delta_t,
		\end{align}
		for $\delta_t= -\E_{p_{t,m}}[l_{t,m}]+\E_u[l_{t,m}]$.
	\end{proof}
\end{proposition}

Using \autoref{thm:ptmqtm} together with \autoref{thm:bound12} gives us the following result.
\begin{lemma}\label{thm:bound13}
	When using \autoref{alg:framework}, we have
	\begin{align*}
		R_T(\Lambda_T)\triangleq&\sum_{t=1}^T\left(\E_{q_{t,m}}[l_{t,m}]-l_{t,\lambda_t(1)}\right)
		\\\leq& \E\left[\sum_{t=1}^T\eta_tv_t\right]+\E\left[\frac{W(\Lambda_T)+A_T}{\eta_T}\right]
		\\&+\sum_{t=1}^{T}\E_{p_{t,m}}[\mu_{t,m}]-\mu_{t,\lambda_t(1)}+\sum_{t=1}^{T}\epsilon_t\delta_t,
	\end{align*}
	where $v_t\triangleq \enspace\E_{p_{t,m}}\phi_{t,m}^2$, $W(\Lambda_T)\triangleq \log(\max_{1\leq t\leq T}|\Omega_{t-1}|)-\log(\Tau(\Lambda_T))$; and we have \autoref{ass:-eta_tphi_tm}, \autoref{ass:etatnoninc}, \autoref{ass:etatdt} and \autoref{ass:phitm}.
	\begin{proof}
		Result follows from combining \autoref{thm:bound12} and \autoref{thm:ptmqtm}.
	\end{proof}
\end{lemma}

The result in \autoref{thm:bound13} provides us with a performance bound with respect to the selection probabilities $q_{t,m}$, i.e., an expected regret bound.
To further simplify the result, we have the following assumptions.
\begin{assumption}\label{ass:utm}
	Let $\mu_{t,m'}=\mu_{t,m}, \forall t, \forall m,m'.$
\end{assumption}

\begin{assumption}\label{ass:WT}
	Let $W(\Lambda_T)\leq W$ and $A_T=W$.
\end{assumption}

\begin{lemma}\label{thm:bound14}
	When using \autoref{alg:framework}, we have
	\begin{align*}
		R_T(\Lambda_T)\leq& \E\left[\sum_{t=1}^T\eta_tv_t\right]+\E\left[\frac{2W}{\eta_T}\right]+\sum_{t=1}^{T}\epsilon_t\delta_t,
	\end{align*}
	where $v_t\triangleq \enspace\E_{p_{t,m}}\phi_{t,m}^2$, $W(\Lambda_T)\triangleq \log(\max_{1\leq t\leq T}|\Omega_{t-1}|)-\log(\Tau(\Lambda_T))$; and we have \autoref{ass:-eta_tphi_tm}, \autoref{ass:etatnoninc}, \autoref{ass:etatdt}, \autoref{ass:phitm} and \autoref{ass:utm}.
	\begin{proof}
		Result follows from a direct application of \autoref{ass:utm} and \autoref{ass:WT} in \autoref{thm:bound13}.
	\end{proof}
\end{lemma}

The result in \autoref{thm:bound14} is fundamentally dependent on the competition complexity bound $W$. To increase comprehension, we provide some examples for the competition complexities of the implementations in \autoref{sec:example}. 
\begin{remark}
	We have the following complexity bounds $W$ for different example implementations in \autoref{sec:example}.
	\begin{itemize}
		\item For fixed bandits: we will have $W=\log M$.
		
		\item For switching bandits with at most $S$ switches, we will have $W=\tilde{O}(S)$. 
		
		\item For contextual bandits with at most $C=\sum_{k=1}^Kc_k$ total regions in an unknown number of distinct $K$ segments and unknown individual $\{c_k\}_{k=1}^K$, we will have $W=\tilde{O}(C)$. 
		
		\item For periodic bandits with at most $\mathit{P}=\sum_{k=1}^{K}\tau_k$ total periods in an unknown number of distinct $K$ segments and unknown individual $\{\tau_k\}_{k=1}^K$, we will have $W=\tilde{O}(P)$. 
	\end{itemize}
\end{remark}

The expected regret in \autoref{thm:bound14} highly depends on the learning rate $\eta_t$, the performance measures $\phi_{t,m}$ and the uniform mixing coefficient $\epsilon_t$; which will be designed in accordance with the problem setting as shown in the next sections.

\section{Full Feedback Setting}\label{sec:full}

We extend the general regret result in \autoref{sec:regret} to various settings. We first tackle the full feedback setting, where, after playing the $i^{th}$ arm at time $t$, we have access to the losses $l_{t,m}$ for all $m$. The parameters that we need to design are the learning rate $\eta_t$, performance measures $\phi_{t,m}$ and the uniform mixing coefficient $\eta_t$ for all $t,m$. 

\begin{remark}\label{rem:ass}
	The selected parameters $\eta_t, \phi_{t,m}$ and $\eta_t$ need to satisfy the following assumptions as per \autoref{sec:regret}:
	\begin{itemize}
		\item Let $-\eta_t\phi_{t,m}\leq 1, \forall t,m$.
		\item Let $n_t$ is nonincreasing with $t$, i.e., $n_{t+1}\leq n_{t}, \forall t$.
		\item Let $\eta_td_t\leq W, \forall t\leq T$.
		\item Let $\phi_{t,m}$ be such that $\E_{i_t\sim q_t}[\phi_{t,m}]=l_{t,m}-\mu_{t}.$
	\end{itemize}
\end{remark}
To satisfy the last item in \autoref{rem:ass}, we can set
	\begin{align}
		\phi_{t,m}=l_{t,m}-\mu_t\label{eq:qtm}
	\end{align}
for any real $\mu_t$. We observe that changing $\mu_t$ has no algorithmic difference since the updated weights are normalized to acquire the selection probabilities.

\begin{remark}\label{rem:sec1}
	We set the parameters $\eta_t,\epsilon_t,\phi_{t,m}$ as the following:
	\begin{itemize}
		\item $\epsilon_t=0$,
		\item $\phi_{t,m}=l_{t,m}-\sum_{m'}p_{t,m'}l_{t,m'}$,
		\item $\displaystyle\eta_t=\min\left(\frac{\sqrt{W}}{\sqrt{V_t}},\frac{W}{D_t},\frac{1}{\Phi_t}\right),$
		
		where $\Phi_t=\min_{t'\leq t}\min_m {\phi_{t,m}}$,
	\end{itemize}
	which satisfy the assumptions in \autoref{rem:ass}.
\end{remark}

\begin{theorem}\label{thm:11}
	When using \autoref{alg:framework} with the selection in \autoref{rem:sec1}, we have
	\begin{align*}
		R_T(\Lambda_T)\leq&2(1+W)\max_{1\leq t\leq T}\triangle_t+2\sqrt{W\sum_{t=1}^T\triangle_t^2},
	\end{align*}
	where $\triangle_t=\max_m l_{t,m}- \min_{m} l_{t,m}$.
	\begin{proof}
		From \autoref{rem:sec1} and \autoref{thm:bound14}, we have
			\begin{align}
			R_T(\Lambda_T)\leq&
			\E\left[\sum_{t=1}^{T}\frac{\sqrt{W}}{\sqrt{V_t}}\left(V_t-V_{t-1}\right)\right]\nonumber
			\\&+\E\left[2W\frac{\sqrt{V_T}}{\sqrt{W}}+2W\frac{D_T}{W}+2W\Phi_T\right]
			\\\leq&
			\E\left[\sum_{t=1}^{T}\frac{\sqrt{W}}{\sqrt{V_t}}(\sqrt{V_t}-\sqrt{V_{t-1}})(\sqrt{V_t}+\sqrt{V_{t-1}})\right]\nonumber
			\\&+\E\left[2{\sqrt{WV_T}}+2{D_T}+2W\Phi_T\right]
			\\\leq&
			\E\left[\sum_{t=1}^{T}\frac{\sqrt{W}}{\sqrt{V_t}}\left(\sqrt{V_t}-\sqrt{V_{t-1}}\right)\left(2\sqrt{V_t}\right)\right]\nonumber
			\\&+\E\left[2{\sqrt{WV_T}}+2{D_T}+2W\Phi_T\right]
			\\\leq&\E[{4\sqrt{WV_T}}+2D_T+2W\Phi_T].
		\end{align}
		Furthermore, we have
		\begin{align}
			\E V_T\leq& \frac{1}{4}\sum_{t=1}^T\triangle_t^2
			\\\E D_T\leq& \max_{1\leq t\leq T} \triangle_t
			\\ \E\Phi_T\leq& \max_{1\leq t\leq T} \triangle_t,
		\end{align}
		where $\triangle_t=\max_m l_{t,m}- \min_{m} l_{t,m}$ is the range of losses at time $t$. Combining these results, we get
		\begin{align}
			R_T(\Lambda_T)\leq&2(1+W)\max_{1\leq t\leq T}\triangle_t+2\sqrt{W\sum_{t=1}^T\triangle_t^2},
		\end{align}
		which concludes the proof.
	\end{proof}
\end{theorem}

While this regret bound is intuitive, it is not a bona fide second order bound since the first part contains a  zeroth order term. We can remedy this by the following selection.

\begin{remark}\label{rem:sec2}
	We set the parameters $\eta_t,\epsilon_t,\phi_{t,m}$ as the following:
	\begin{itemize}
		\item $\epsilon_t=0$,
		\item $\phi_{t,m}=l_{t,m}-\min_{m'}l_{t,m'}$,
		\item $\displaystyle\eta_t=\min\left(\frac{\sqrt{W}}{\sqrt{V_t}},\frac{W}{D_t}\right),$
	\end{itemize}
	which satisfy the assumptions in \autoref{rem:ass}.
\end{remark}

\begin{theorem}\label{thm:12}
	When using \autoref{alg:framework} with the selection in \autoref{rem:sec2}, we have
	\begin{align*}
		R_T(\Lambda_T)\leq&6\sqrt{W\sum_{t=1}^T\triangle_t^2}
	\end{align*}
\begin{proof}
	From \autoref{rem:sec2} and \autoref{thm:bound14}, after a similar analysis as in the proof of \autoref{thm:11}, we have
	\begin{align}
		R_T(\Lambda_T)\leq& \E[{4\sqrt{WV_T}}+2D_T].
	\end{align}
	Furthermore, we have
	\begin{align}
		\E V_T\leq& \sum_{t=1}^T\triangle_t^2
		\\\E D_T\leq& \max_{1\leq t\leq T} \triangle_t,
	\end{align}
	which results in
	\begin{align}
		R_T(\Lambda_T)\leq&2\max_{1\leq t\leq T}d_t+4\sqrt{W\sum_{t=1}^Td_t^2},
		\\\leq&6\sqrt{W\sum_{t=1}^T\triangle_t^2},
	\end{align}
	which concludes the proof.
\end{proof}
\end{theorem}

\begin{corollary}
	When using \autoref{alg:framework} with the selection in \autoref{rem:sec2}, we have
	\begin{align*}
		R_T(\Lambda_T)\leq&6\triangle\sqrt{WT},
	\end{align*}
	where $\triangle_t\geq \triangle$ for all $t\leq T$.
	\begin{proof}
		The result is a direct application of \autoref{thm:12}, when the loss ranges are uniformly bounded.
	\end{proof}
\end{corollary}

\section{Bandit Feedback Setting}\label{sec:bandit}
In this section, we extend the general regret result in \autoref{sec:regret} to the bandit feedback setting, where, after playing the $i^{th}$ arm at time $t$, we have only access to the loss $l_{t,i}$. Similarly with \autoref{sec:full}, the parameters that we need to design are the learning rate $\eta_t$, performance measures $\phi_{t,m}$ and the uniform mixing coefficient $\mu_t$ for all $t,m$. 

\begin{remark}\label{rem:ass2}
	The selected parameters $\eta_t, \phi_{t,m}$ and $\eta_t$ need to satisfy the following assumptions as per \autoref{sec:regret}:
	\begin{itemize}
		\item Let $-\eta_t\phi_{t,m}\leq 1, \forall t,m$.
		\item Let $n_t$ is nonincreasing with $t$, i.e., $n_{t+1}\leq n_{t}, \forall t$.
		\item Let $\eta_td_t\leq W, \forall t\leq T$.
		\item Let $\phi_{t,m}$ be such that $\E_{i_t\sim q_t}[\phi_{t,m}]=l_{t,m}-\mu_{t}.$
	\end{itemize}
\end{remark}

Unlike \autoref{sec:full}, algorithmic invariance with respect to $\mu_t$ is not present because of the limited bandit feedback.

\begin{remark}\label{rem:sec1b}
	We set the parameters $\eta_t,\epsilon_t,\phi_{t,m}$ as the following:
	\begin{itemize}
		\item $\displaystyle\epsilon_t=\min\left(\frac{1}{2},\frac{\sqrt{MW}}{\sqrt{t}}\right)$,
		\item $\displaystyle\phi_{t,m}=\mathbbm{1}_{i_t=m}\frac{l_{t,m}-l_{t-1,i_{t-1}}}{q_{t,m}}$,
		\item $\displaystyle\eta_t=\min\left(\frac{\sqrt{W}}{\sqrt{V_t}},\frac{1}{D_t}\right),$
	\end{itemize}
	which satisfy the assumptions in \autoref{rem:ass2}.
\end{remark}

\begin{theorem}\label{thm:21}
	When using \autoref{alg:framework} with the selection in \autoref{rem:sec1b}, we have
	\begin{align*}
		R_T(\Lambda_T)\leq&\tilde{O}\left(MW\sum_{\tau=1}^{\lceil\sqrt{\frac{T}{MW}}\rceil}\tilde{\triangle}'_\tau\right),
	\end{align*}
	where $\tilde{O}(\cdot)$ is the soft-O notation, which ignores logarithmic terms; $\tilde{\triangle}'_\tau$ is the $\tau^{th}$ largest range in $\{\tilde{\triangle}_t\}_{t=1}^T$; and ${\tilde{\triangle}_t=\max_m \max (l_{t,m},l_{t-1,m})-\min_m \min (l_{t,m},l_{t-1,m})}$.
	\begin{proof}
		From \autoref{rem:sec1b} and \autoref{thm:bound14}, after a similar analysis as in the proof of \autoref{thm:11}, we have
		\begin{align}
			R_T(\Lambda_T)\leq& \E\left[2\sqrt{WV_T}\right]+\E\left[2\sqrt{WV_T}+2WD_T\right]\nonumber
			\\&+\sum_{t=1}^{T}\epsilon_t\delta_t.
		\end{align}
		Hence, we have
		\begin{align}
			R_T(\Lambda_T)\leq& \E\left[2\sqrt{WV_T}\right]+\E\left[2\sqrt{WV_T}+2WD_T\right]\nonumber
			\\&+\E\left[\sqrt{\sum_{t=1}^{T}\epsilon_t^2}\sqrt{\sum_{t=1}^{T}\delta_t^2}\right]
			\\\leq&\E\left[4\sqrt{WV_T}\right]+\E\left[2WD_T\right]\nonumber
			\\&+\E\left[\sqrt{MW\log(eT)\sum_{t=1}^{T}\delta_t^2}\right].
		\end{align} 
		Furthermore, we have
		\begin{align}
			\delta_t\leq& \triangle_t
			\\\E V_T\leq& 2M\sum_{t=1}^T\tilde{\triangle}_t^2
		\end{align}
		where $\triangle_t=\max_m l_{t,m}-\min_m l_{t,m}$ is the loss range at time $t$ and $\tilde{\triangle}_t=\max_m \max (l_{t,m},l_{t-1,m})-\min_m \min (l_{t,m},l_{t-1,m})$ is the extended loss range at times $t-1,t$. Thus, we reach
		\begin{align}
			R_T(\Lambda_T)\leq&4\sqrt{2MW\sum_{t=1}^{T}\tilde{\triangle}_t^2}+\E\left[2WD_T\right]
			\nonumber\\&+\sqrt{MW\log(eT)\sum_{t=1}^{T}\tilde{\triangle}_t^2}
			\\\leq& \tilde{O}\left(\sqrt{MW\sum_{t=1}^{T}\tilde{\triangle}_t^2}\right)+2W\E\left[D_T\right],
		\end{align}
		where $\tilde{O}(\cdot)$ is the soft-O notation, which ignores logarithmic terms.
		To bound the expectation of $D_T$, we do the following. Let
		$\tilde{q}_{t,m}$ be the probability that $D_T$ is equal to $d_t$ which is again equal to $\frac{l_{t,m}-l_{t-1,i_{t-1}}}{q_{t,m}}$, i.e.,
		\begin{align}
			\tilde{q}_{t,m}=\mathbb{P}(D_T=d_t=\phi_{t,m}).
		\end{align}
	Then, we have the following expected bound
		\begin{align}
			\E[D_T]\leq& \sum_{t,m}\tilde{q}_{t,m}\frac{\tilde{\triangle}_t}{q_{t,m}}
		\end{align}
		for $\tilde{q}_{t,m}\leq q_{t,m}$ and $\sum_{t,m}\tilde{q}_{t,m}=1$. Since this function is convex and $\tilde{\triangle}_t$ are constant over $m$, we have
		\begin{align}
			\E[D_T]\leq&M\sum_{\tau=1}^{\lceil\epsilon_T^{-1}\rceil}\tilde{\triangle}'_\tau, 
		\end{align}
		where $\tilde{\triangle}'_\tau$ is the $\tau^{th}$ largest range in $\{\tilde{\triangle}_t\}_{t=1}^T$.
		Hence, we get
		\begin{align*}
			R_T(\Lambda_T)\leq&\tilde{O}\left(\sqrt{MW\sum_{t=1}^{T}\tilde{\triangle}_t^2}\right)+O\left(MW\sum_{\tau=1}^{\lceil\sqrt{\frac{T}{MW}}\rceil}\tilde{\triangle}'_\tau\right),
		\end{align*}
		which concludes the proof.
	\end{proof}
\end{theorem}

\begin{corollary}
	When using \autoref{alg:framework} with the selection in \autoref{rem:sec1b}, we have
	\begin{align*}
		R_T(\Lambda_T)\leq&\tilde{O}\left(\tilde{\triangle}\sqrt{MWT}\right),
	\end{align*}
	where $\tilde{\triangle}\geq \tilde{\triangle}_t$ for all $t\leq T$.
	\begin{proof}
		The result is a direct application of \autoref{thm:21}, when the loss ranges are uniformly bounded.
	\end{proof}
\end{corollary}


\bibliographystyle{ieeetran}
\bibliography{double_bib}	
	





	
\end{document}